\documentclass{article}
\usepackage{ijcai24}
\usepackage{soul}
\usepackage[switch]{lineno}
\usepackage[utf8]{inputenc} 
\usepackage{times}
\usepackage{url}            
\usepackage{booktabs}       
\usepackage{amsfonts}       
\usepackage{nicefrac}       
\usepackage{microtype}      
\usepackage{graphicx}
\usepackage{subfigure}

\usepackage{amsmath}
\usepackage{amssymb}
\usepackage{mathtools}
\usepackage{amsthm}

\usepackage[capitalize,noabbrev]{cleveref}

\theoremstyle{plain}
\newtheorem{theorem}{Theorem}[section]
\newtheorem{proposition}[theorem]{Proposition}

\theoremstyle{definition}

\theoremstyle{remark}

\usepackage{amsthm}
\usepackage{color}
\usepackage[para,online,flushleft]{threeparttable}
\usepackage{enumerate}
\usepackage{multirow}
\newcommand{\argmax}{\mathop{\rm arg~max}\limits}

\usepackage{algorithm}
\usepackage{algorithmic}
\renewcommand{\algorithmicrequire}{\textbf{Input:}}
\renewcommand{\algorithmicensure}{\textbf{Output:}}
\newcommand{\textbfit}[1]{\textbf{\textit{#1}}}

\newcommand{\mE}{\mathbb{E}}
\newcommand{\mV}{\mathbb{V}}
\newcommand{\mP}{\mathbb{P}}
\newcommand{\calD}{\mathcal{D}}
\newcommand{\calDtr}{\calD_{tr}}
\newcommand{\calDval}{\calD_{val}}
\newcommand{\calX}{\mathcal{X}}
\newcommand{\calA}{\mathcal{A}}
\newcommand{\calS}{\mathcal{S}}

\newcommand{\trueV}{V(\pi_e)}
\newcommand{\ips}{\hat{V}_{\mathrm{IPS}} (\pi_e; \calD)}
\newcommand{\ipssimple}{\hat{V}_{\mathrm{IPS}}}
\newcommand{\snips}{\hat{V}_{\mathrm{SNIPS}} (\pi_e; \calD)}

\newcommand{\dr}{\hat{V}_{\mathrm{DR}} (\pi_e; \calD, \hat{\mu})}
\newcommand{\ipsdiff}{\Delta \hat{V}_{\mathrm{IPS}}\left(\pi_{1}, \pi_{2}\right)}

\title{
Hyperparameter Optimization Can Even be Harmful\\ in Off-Policy Learning and How to Deal with It
}

\author{
Yuta Saito$^1$
\and
Masahiro Nomura$^2$
\affiliations
$^1$Cornell University,
$^2$CyberAgent, Inc.
\emails
ys552@cornell.edu,
nomura\_masahiro@cyberagent.co.jp
}

\begin{document}

\maketitle

\begin{abstract}
There has been a growing interest in off-policy evaluation in the literature such as recommender systems and personalized medicine. We have so far seen significant progress in developing estimators aimed at accurately estimating the effectiveness of counterfactual policies based on biased logged data. However, there are many cases where those estimators are used not only to evaluate the value of decision making policies but also to search for the best hyperparameters from a large candidate space. This work explores the latter \textbfit{hyperparameter optimization (HPO)} task for \textit{off-policy learning}. We empirically show that naively applying an unbiased estimator of the generalization performance as a surrogate objective in HPO can cause an unexpected failure, merely pursuing hyperparameters whose generalization performance is greatly overestimated. We then propose simple and computationally efficient corrections to the typical HPO procedure to deal with the aforementioned issues simultaneously. Empirical investigations demonstrate the effectiveness of our proposed HPO algorithm in situations where the typical procedure fails severely.
\end{abstract}

\section{Introduction}
Interactive decision making systems, such as recommender systems, produce logged data valuable for optimizing future decision making.
For example, the logs of an e-commerce recommender system record which product was recommended and whether the users purchased it, giving the system designer a rich logged dataset useful for evaluating and improving the decision making quality. 
This type of historical data is often called \textit{logged bandit data} and is one of the most ubiquitous forms of data available in many real-life applications~\cite{swaminathan2015batch,su2020doubly,kiyohara2021accelerating,saito2024potec}.
 
\textit{Off-Policy Learning (OPL)} aims to train a new decision making policy using only the logged bandit data. 
OPL is useful in that it can improve the decision making system continuously in a batch manner without requiring a risky exploration.
Owing to the ubiquity of logged bandit data in the real-world, significant attention has been paid to OPL of contextual bandits~\cite{strehl2010learning,swaminathan2015batch,swaminathan2015self,wang2016optimal,kallus2021optimal,kiyohara2023towards,kiyohara2024off}.

The fundamental problem in OPL is that the outcome is only observed for the action chosen by the system in the past. 
Thus, estimating the generalization performance of a policy is non-trivial because we cannot naively apply the empirical risk as done in typical supervised machine learning (ML).
Therefore, a variety of estimators have been developed in the field of off-policy evaluation (OPE), such as Inverse Propensity Score (IPS)~\cite{precup2000eligibility} and Doubly Robust (DR)~\cite{dudik2014doubly}.
Then, a feasible approach to OPL is to maximize one such estimator as a surrogate objective using only the logged data. Hyperparameter optimization (HPO) can also be performed based on one of the estimators on a validation set of the logged data~\cite{paine2020hyperparameter}.

In this study, we investigate how well automatic HPO algorithms work for OPL using only the available logged data. In particular, we empirically find two critical issues in HPO that have yet to be investigated in the literature, but can have a significant adverse impact on the effectiveness of the OPL pipeline. The first issue is \textbfit{optimistic bias}, which implies that the hyperparameter values selected by an HPO procedure are often the ones whose performance is greatly overestimated. In HPO, we often use an unbiased estimator as a strategy to optimize the generalization performance (primary objective) using only validation data. The problem is that, when optimizing the validation performance as a surrogate objective, HPO can identify a set of hyperparameters whose validation performance looks good but its generalization performance is detrimental. As a result, the typical HPO procedure often produces a highly sub-optimal solution, even with an unbiased estimator of the generalization performance. The second issue is \textbfit{unsafe behavior}, which suggests that the typical HPO procedure can output a solution, which \textit{underperforms} the logging (data collection) policy, even when we set the logging policy as an initial solution. This is problematic because a logging policy is often a baseline policy to improve upon in OPL. If an HPO procedure aggravates the performance of the logging policy, there is no need to implement it in practice.

After formulating the problem in Section~\ref{sec:setup}, Section~\ref{sec:analysis} provides clear empirical evidence of optimistic bias and unsafe behavior. We observe these phenomena even when we use an unbiased surrogate objective and a popular adaptive HPO algorithm. We also explain these observations theoretically, demonstrating that ignoring the fact that HPO optimizes the validation performance as a surrogate of the generalization performance can lead to a worse regret of HPO algorithms. More specifically, we identify that a heavy-tailed distribution of overestimation bias during HPO can cause an unexpected gap between the generalization and validation regret. These empirical and theoretical observations result in our proposed corrections to the typical HPO procedure, which we describe in Section~\ref{sec:proposal}. Finally, Section~\ref{sec:experiment} conducts comprehensive experiments and demonstrates that our simple corrections can deal with the aforementioned issues and improve the typical procedure, particularly for cases where the typical procedure becomes unsafe and underperforms the logging policy.\footnote{Appendix B provides a comprehensive survey of related work.}

\section{Preliminaries}
\label{sec:setup}
We use $x \in \calX$ to denote a context vector and $a \in \calA$ to denote a (discrete) action such as a playlist recommendation in a music streaming service. Let $r \in [0, r_{\mathrm{max}}]$ denote a reward variable, which is sampled identically and independently from an unknown conditional distribution $p (r | x, a)$. A decision making policy is modeled as a distribution over the action space, i.e., $\pi: \calX \rightarrow \Delta(\calA)$ where $\Delta(\cdot)$ is a probability simplex. We can then represent the probability of action $a$ being taken by policy $\pi$ given context $x$ as $\pi (a | x)$.

\subsection{Off-Policy Evaluation and Learning}
In OPE, we are given logged bandit data $\calD := \{(x_i,a_i,r_i)\}_{i=1}^n$ consisting of $n$ independent draws from the \textit{logging policy} $\pi_0$. Using this logged dataset, OPE aims to estimate the \textit{generalization} performance of a given \textit{evaluation policy} $\pi_e$, which is often different from $\pi_0$: 
\begin{align}
    \trueV := \mE_{(x,a,r) \sim p(x) \pi_e (a | x) p(r | x, a)} [r]. \label{eq:policy_value}
\end{align}

This is the ground-truth performance of the evaluation policy when deployed in an environment of interest. OPE uses an estimator $\hat{V}$ to estimate $\trueV$ based only on $\calD$ as $\trueV \approx \hat{V} (\pi_e; \calD)$. A typical choice of $\hat{V}$ is IPS:
\begin{align*}
    \ips := \frac{1}{n} \sum_{i=1}^n \frac{\pi_e(a_i \,|\, x_i)}{\pi_0 (a_i \,|\, x_i)} r_i, 
\end{align*}
where $\pi_e(a_i | x_i)/\pi_0(a_i | x_i)$ is called the importance weight. Under some assumptions for identification such as full support ($ \pi_e(a|x) > 0 \rightarrow \pi_0(a|x) > 0, \; \forall (x,a)$), IPS provides an unbiased estimate of the generalization policy performance, i.e., $\mE_{\calD} [\ips] = \trueV $. Beyond IPS, significant efforts have been made to enable a more accurate OPE from the logged data~\cite{dudik2014doubly,wang2016optimal,su2020doubly,saito2023off}.

In OPL, we aim to learn an optimal decision making policy $\pi^* := \argmax_{\pi} V(\pi)$ from the logged data. As in supervised ML, we cannot directly use the generalization policy performance. Instead, we use its estimator as a surrogate:
\begin{align*}
    \hat{\pi} = \argmax_{\pi \in \Pi} \; \hat{V}(\pi; \calD) - \lambda \cdot \mathcal{R} (\pi),
\end{align*}
where $\Pi$ is a policy class, which might be a linear class~\cite{swaminathan2015batch} or deep neural nets~\cite{joachims2018deep}. $\mathcal{R}(\cdot)$ regularizes the complexity of the policy $\pi$, and $\lambda (\ge 0)$ is a hyperparamter that controls the effect of regularization.

\begin{algorithm}[t]
\caption{Typical HPO with IPS as a surrogate (\textbf{Baseline})}
\label{algo:naive_hpo}
\begin{algorithmic}[1]
\REQUIRE $A$, $\Theta$, $T$, $\calDtr$, $\calDval$
\ENSURE $\hat{\theta}$
\STATE $\pi^* \leftarrow \pi_0$ , $\calS_0 \leftarrow \emptyset$
\FOR{$t = 1,2, \ldots, T$}
    \STATE $\theta_t \leftarrow A(\theta \,|\, \calS_{t-1}) $  \hfill \textit{\small // sample candidate hyperparameters}
    \STATE $\pi_t \leftarrow \hat{\pi}(\cdot \,|\, \theta_t, \calDtr) $ \hfill \textit{\small // train a policy (lower-level)}
    \IF{$\ipssimple \left(\pi_t; \calDval\right) > \ipssimple \left(\pi^*; \calDval\right)$}
        \STATE $\hat{\theta} \leftarrow \theta_t$, $\pi^* \leftarrow \pi_t$ \hfill \textit{\small // update the solution}
    \ENDIF
    \STATE $\calS_t \leftarrow \calS_{t-1} \cup \{(\theta_t, \ipssimple(\pi_t;\calDval))\}$ \hfill \textit{\small // store the result}
\ENDFOR
\end{algorithmic}
\end{algorithm}

\subsection{Hyperparameter Optimization}
OPE involves many hyperparameters to be properly tuned from those defining the policy class $\Pi$ to the regularization parameter $\lambda$. In a typical HPO procedure for OPL, we first split the original logged bandit data $\calD$ into training ($\calDtr$) and validation ($\calDval$) sets. Then, we wish to solve the following bi-level optimization:
\begin{align}
    \theta^* := \argmax_{\theta \in \Theta} \; V \big(\hat{\pi}(\cdot; \theta, \calDtr) \big), \label{eq:opt_hypara}
\end{align}
where $\Theta$ is a pre-defined hyperparamter search space. $\hat{\pi}(\cdot; \theta, \calDtr)$ is a policy parameterized by a set of hyperparameters $\theta$. The model parameter of $\hat{\pi}(\cdot; \theta, \calDtr)$ is trained on the training set $\calDtr$ (lower-level optimization). The problem here is that the generalization performance of $\hat{\pi}(\cdot; \theta, \calDtr)$ is unknown and needs to be estimated. A feasible HPO procedure based on an estimated policy performance is:
\begin{align}
    \hat{\theta}(\calDval) := \argmax_{\theta \in \Theta} \; \hat{V} \big(\hat{\pi}(\cdot; \theta, \calDtr); \calDval \big), \label{eq:hpo}
\end{align}
where the generalization performance of $\hat{\pi}(\cdot; \theta, \calDtr)$ is estimated by an estimator $\hat{V}$ on the validation set $\calDval$.\footnote{For brevity of notation, we sometimes use $V(\theta)$ and $\hat{V}(\theta;\calD)$ to denote the generalization and validation performances of the policy induced by $\theta$.} A common choice of $\hat{V}$ is an unbiased estimator that satisfies $\mE [ \hat{V} \left(\pi; \calDval \right) ] = V(\pi), \forall \pi \in \Pi$ such as IPS. Then, one can apply grid search, random search~\cite{bergstra2012random}, or adaptive methods such as tree-structured Parzen estimator (TPE)~\cite{bergstra2011algorithms} to solve the higher-level optimization in Eq.~\eqref{eq:hpo} efficiently. Algorithm~\ref{algo:naive_hpo} describes this typical HPO procedure for OPL, which starts from the logging policy $\pi_0$ as its initial solution and adaptively samples promising hyperparameters via an arbitrary HPO algorithm (denoted here as $A$)~\cite{tang2021model}.

\begin{figure*}[ht]
\includegraphics[width=16.5cm]{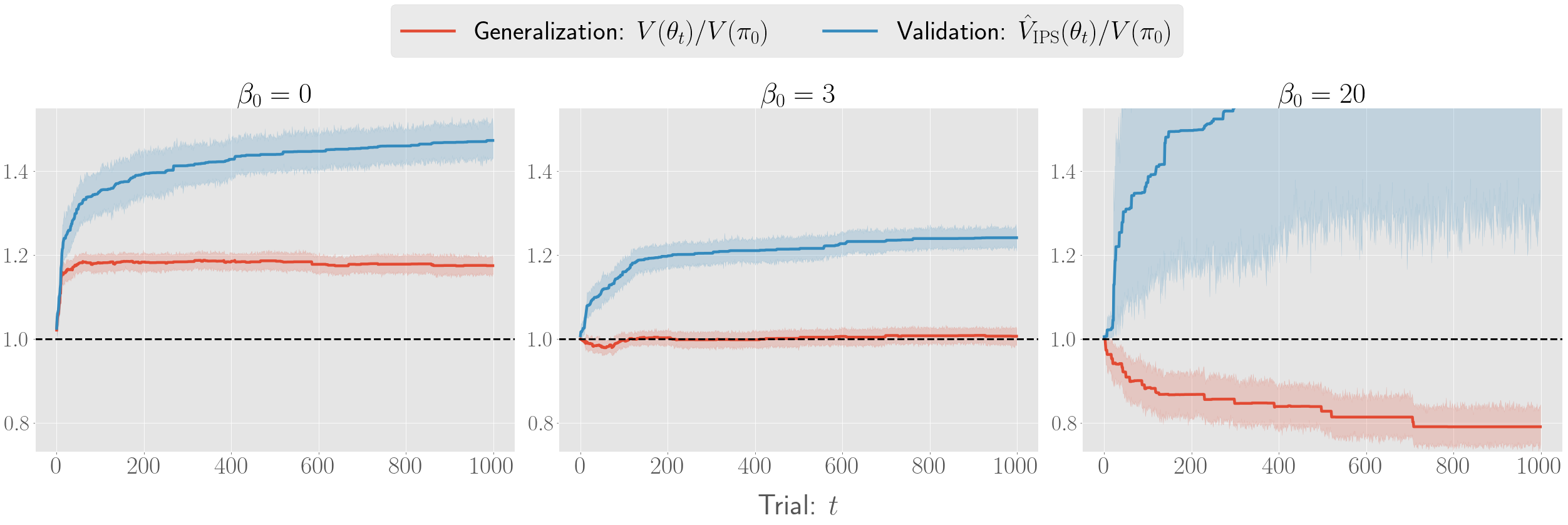}
\vspace{-2mm}
\caption{Empirical Evidence of Optimistic Bias and Unsafe Behavior in HPO for OPL (w/ TPE). The results are averaged over 25 runs with different seeds and then normalized by $V(\pi_0)$. The shaded regions indicate 95\% confidence intervals.}
\label{fig:optimistic_bias_tpe}
\end{figure*}

\section{Unexpected Failure in HPO for OPL}\label{sec:analysis}
This section studies the effectiveness of HPO when applied to OPL from both empirical and theoretical perspectives.

\subsection{Empirical Analysis}\label{sec:empirical_analysis}
First, we conduct a synthetic experiment and provide empirical evidence of surprising failure of HPO in OPL.

\subsubsection{Synthetic Data.}
Our empirical analysis is based on \textit{OpenBanditPipeline} (OBP)\footnote{https://github.com/st-tech/zr-obp}, an open-source toolkit for OPE and OPL, which includes synthetic data generation modules and a range of estimators~\cite{saito2021open}. We synthesize context vectors $x$ by sampling them from a 10-dimensional standard normal distribution. We then set $|\calA|=10$, where each action $a \in \calA$ is characterized by a 10-dimensional representation vector $e_a$. The reward function $\mu(x,a):=\mE[r \,|\, x,a]$ is defined as:
\begin{align}
    \mu(x,a) = \sigma \left( x^{\top} M e_a + \eta_x^{\top} x +  \eta_a^{\top} e_a \right), \label{eq:synthetic_reward}
\end{align}
where $\sigma(z):= 1/(1+\exp(-z))$ is the sigmoid function. $M$, $\eta_x$, and $\eta_a$ are parameter matrices or vectors for defining the synthetic reward function. These parameters are sampled from a uniform distribution with range $[-1,1]$. After generating the synthetic reward function, we sample binary rewards from a Bernoulli distribution with parameter $\mu(x,a)$.

We then define the logging policy $\pi_0$ by applying the softmax function to the reward function $\mu(x,a)$ as follows.
\begin{align}
    \pi_0(a \,|\, x) =  \frac{\exp( \beta_0 \cdot \mu(x,a))}{ \sum_{a' \in \calA} \exp( \beta_0 \cdot \mu(x,a')) } , \label{eq:synthetic_logging}
\end{align}
where $\beta_0$ is an inverse temperature parameter to control the optimality and entropy of the logging policy.
A large positive value of $\beta_0$ leads to a near-deterministic and near-optimal logging policy. When $\beta_0=0$, $\pi_0$ is uniform.

\subsubsection{Policy Class and HPO Algorithms.}
To train a new policy $\pi$ from only the logged data, we first estimate $\mu(x,a)$ by a supervised ML method, where the resulting estimator is denoted as $\hat{\mu}(x,a;\calDtr)$. We then form a stochastic policy by applying the softmax rule as:
\begin{align}
    \pi(a \,|\, x; \theta, \calDtr) =  \frac{\exp( \beta \cdot \hat{\mu}(x,a;\calDtr))}{ \sum_{a' \in \calA} \exp( \beta \cdot\hat{\mu}(x,a';\calDtr)) } , \label{eq:synthetic_eval}
\end{align}
where $\beta$ is an inverse temperature parameter to define a new policy. $\theta$ is a set of hyperparameters, which consists of $\beta$, supervised ML model to construct $\hat{\mu}$, and the hyperparameters of $\hat{\mu}$. The hyperparameter search space $\Theta$ is summarized in Table 1 in Appendix E.

As an HPO algorithm, we use TPE~\cite{bergstra2011algorithms}, which is a popular adaptive method in the HPO community~\cite{akiba2019optuna}. TPE has been shown to work well for HPO of supervised ML, however, whether it also works for OPL has never been thoroughly investigated.

\subsubsection{Observations.}
In this synthetic experiment, we set $\beta_0 \in \{0,3,20\}$ and $|\calDtr|=|\calDval|=1,000$. The number of trials ($T$ in Algorithm~\ref{algo:naive_hpo}) for HPO is set to 1,000.

Figure~\ref{fig:optimistic_bias_tpe} shows the \textbf{validation performance} ($\ipssimple(\pi; \calDval)$; what HPO algorithm maximizes from the logged data) and the \textbf{generalization performance} ($V(\pi)$; the primary objective of OPL) during the HPO procedure. We obtain the following key observations in this experiment.

\begin{enumerate}
    \item \textbf{Optimistic Bias}: For all $\beta_0$, TPE succeeds in maximizing the validation performance, monotonically improving the blue lines. However, there is a substantial gap between validation and generalization, and the validation performance becomes an extremely optimistic proxy of the generalization performance. For example, when $\beta_0=3$, TPE does not bring any impact on the generalization performance, although the validation performance is greatly improved. This result suggests that implementing HPO is indeed a waste of time and resources for this setting. \vspace{0.025in}
    \item \textbf{Unsafe Behavior}: When $\beta_0 = 20$ (where $\pi_0$ is already much better than uniform random), TPE outputs a solution that is significantly worse than the logging policy with respect to the generalization performance. This is problematic, as the solution at the final trial seems to provide a substantial improvement over the logging policy with respect to the unbiased validation performance (blue lines). In reality, we have no access to the generalization performance (red lines), making it impossible to detect this performance degradation, possibly deploying an unsafe policy in the field without even noticing it. 
\end{enumerate}

These observations suggest that optimizing an unbiased surrogate objective is not an ideal strategy and is even harmful in some cases regarding the optimization of the generalization performance. Note that we obtain similar results when random search (RS) is used as an HPO algorithm and DR is used as an OPE estimator as reported in Appendix E. In particular, comparing RS with TPE in terms of the generalization performance, we find that there are no particular differences between the two algorithms for $\beta_0=0,3$. Even more surprisingly, when $\beta_0=20$, TPE is outperformed by RS, even if TPE is better at optimizing the validation performance. These results further suggest that merely optimizing an unbiased surrogate objective is not a suitable approach for optimizing the generalization performance in HPO of OPL.

\subsection{Theoretical Analysis}\label{sec:theoretical_analysis}
Next, we investigate the mechanism causing the somewhat surprising issues observed in the previous section.\footnote{Appendix C provides proofs omitted in the main text.}
First, we explain the phenomena from a statistical perspective.
\begin{proposition} \label{prop:optimistic_bias}
Given that $\hat{V}$ is unbiased, we have the following inequalities. 
\begin{align}
    \mE_{\calD} \big[\hat{V} \big(\hat{\theta}(\calD); \calD \big) \big] 
     \ge V \left( \theta^* \right) \ge \mE_{\calD} \big[ V \big( \hat{\theta}(\calD) \big) \big], \label{eq:optimistic_bias}
\end{align}
where $\mE_{\calD}[\cdot]$ takes expectation over every randomness in the logged data $\calD$, and $\mE_{\calD} [\hat{V} (\hat{\theta}(\calD); \calD) ] -  \mE_{\calD} [ V ( \hat{\theta}(\calD) ) ]$ is the amount of optimistic bias.
\end{proposition}
Note that, in Eq.~\eqref{eq:opt_hypara}, $V \left( \theta^* \right)$ is defined as the best generalization performance we could achieve with HPO. Thus, the first inequality in Eq.~\eqref{eq:optimistic_bias} suggests that the \textit{validation} performance of the HPO solution $\hat{\theta}(\calDval) $ is better than the best achievable generalization performance in expectation, suggesting that the performance estimation of the HPO solution is optimistic in general. In addition, the second inequality in Eq.~\eqref{eq:optimistic_bias} implies that the \textit{generalization} performance of $\hat{\theta}(\calDval) $ is worse than the best achievable generalization performance in expectation, even though the validation performance of $\hat{\theta}(\calDval) $ is likely to be better. As a result, we will often be disappointed with the performance of the HPO solution $\hat{\theta}$ even with an unbiased surrogate (validation) objective. Overall, Proposition~\ref{prop:optimistic_bias} explains the substantial gap between the blue ($\mE [\hat{V} (\hat{\theta}(\calDval) ) ] $) and red ($ \mE [ V ( \hat{\theta}(\calDval) )] $) lines observed in Figure~\ref{fig:optimistic_bias_tpe}.

\begin{figure}[th]
\centering
\includegraphics[clip, width=7.5cm]{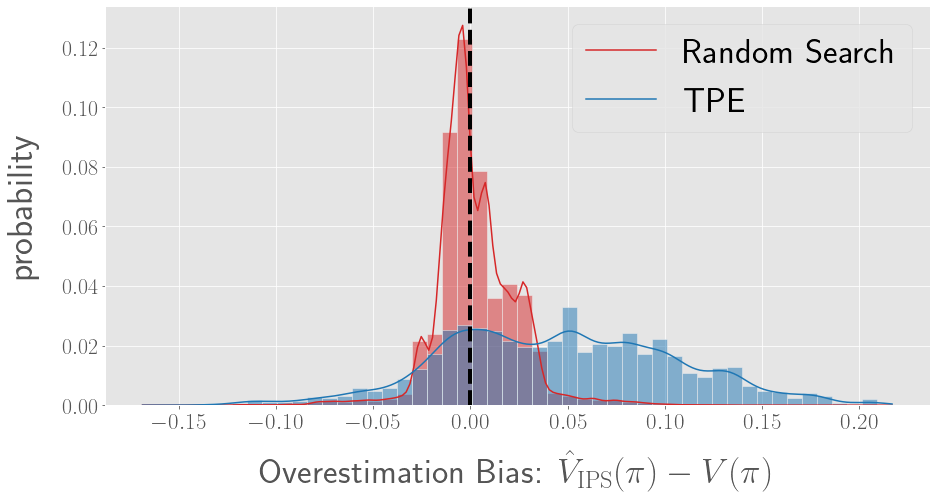}
\vspace{-2mm}
\caption{Distributions of Overestimation Bias ($\beta_0=3$)}
\label{fig:overestimation_bias}
\end{figure}

Next, we analyze ``regret" to understand \textit{what causes the optimistic bias in Proposition~\ref{prop:optimistic_bias} and how we can deal with it}. For this, we define two variants of regret, which measure the difference between the validation or generalization performances of the optimal hyperparameter and HPO solution.
\begin{align}
    r_{gen}(T; A, \calD) & := V \left(\theta^*\right) - V ( \hat{\theta}_{T,A} (\calD) ) \label{eq:generalization_regret},\\
    r_{val}(T; A, \calD) & := \ipssimple (\hat{\theta}^*; \calD) - \ipssimple ( \hat{\theta}_{T,A} (\calD); \calD) \label{eq:validation_regret},
\end{align}
where $\theta^* := \argmax_{\theta \in \Theta} V (\theta)$ is the optimal hyperparameter with respect to the generalization performance. $\hat{\theta}^* = \argmax_{\theta \in \Theta} \ipssimple(\theta; \calD)$ denotes the optimal hyperparameter with respect to the \emph{validation} performance, and $\hat{\theta}_{T,A} (\calD)$ is the solution of Algorithm~\ref{algo:naive_hpo} given budget $T$ and algorithm $A$. We also define the overestimation bias for a specific hyperparameter $\theta$ as $\tau (\theta; \calD) = \ipssimple(\theta; \calD) -  V(\theta)$. Then, the following implies that a \textbf{heavy-tailed distribution of overestimation bias during HPO can produce an unexpected gap between the generalization and validation regret}.
\begin{proposition} \label{prop:regret}
    Given HPO algorithm $A$, budget $T$, and logged data $\calD$, the generalization regret can be written as
    \begin{align}
        r_{gen}(T; A, \calD)  
        &= r_{val}(T; A, \calD) + \Delta \tau(\hat{\theta}_{T,A} (\calD), \theta^*; \calD) + C, \label{eq:regret}
    \end{align}
    where $\Delta \tau(\theta_1, \theta_2; \calD) := \tau(\theta_1; \calD) - \tau(\theta_2; \calD)$, and $C := \ipssimple(\theta^*; \calD) - \ipssimple(\hat{\theta}^*; \calD)$.
\end{proposition}

Only the first two terms of the RHS in Eq.~\eqref{eq:regret} depend on the HPO solution $\hat{\theta}_{T,A} (\calD)$, and are thus critical for analyzing the HPO performance. The first term $r_{val}$ is the validation regret. Under some mild conditions, we can achieve \textit{no-regret} ($r_{val}(T; A,\calD) = o(1)$) with optimization methods such as GP-UCB~\cite{srinivas2010gaussian}, as we can target the validation performance directly using available data. The second term $\Delta \tau(\hat{\theta}_{T,A} (\calD), \theta^*; \calD)$ is the difference in the extent of overestimation between $\hat{\theta}_{T,A} (\calD)$ and $\theta^*$. When the extent of overestimation of $\hat{\theta}_{T,A} (\calD)$ is larger than that of $\theta^*$, $\Delta \tau(\hat{\theta}_{T,A} (\calD), \theta^*; \calD)$ becomes large. Therefore, Proposition~\ref{prop:regret} suggests that \textbf{the overestimation bias of $\hat{\theta}_{T,A} (\calD)$ can exacerbate the generalization regret of HPO algorithms}. More specifically, if an HPO algorithm is likely to sample many hyperparameters whose performance is overestimated ($\ipssimple(\theta) - V(\theta) > 0$) and the overestimation bias has a heavy-tailed distribution, the second term of Eq.~\eqref{eq:regret} tends to become large, so does the generalization regret $r_{gen}$.
Given this regret analysis, we investigate the distributions of overestimation bias observed in the empirical analysis in Figure~\ref{fig:overestimation_bias}.
This figure implies that TPE more frequently samples hyperparameters incurring a large overestimation bias than RS. According to Proposition~\ref{prop:regret}, this is why we do not find the advantage of TPE with respect to the generalization performance.
RS has a worse validation regret than TPE, while overestimation bias of RS is not very problematic compared to TPE. As a result, RS performs similarly to or slightly better than TPE in terms of the generalization performance.
In this way, the heavy-tailed distribution of overestimation bias makes the generalization regret of HPO algorithms (in particular TPE) worse than its validation regret, resulting in optimistic bias and possibly unsafe behavior.

\section{How Should We Deal with the Issues?}\label{sec:proposal}
In this section, we propose two simple corrections, namely (i) \textbf{conservative surrogate objective} and (ii) \textbf{adaptive imitation regularization}, to deal with the critical issues in HPO. We also describe the resulting HPO procedure, which we call \textbf{\textit{Conservative and Imitation-Regularized HPO (CIR-HPO)}}.

\subsection{Conservative Surrogate Objective (CSO)}
First, we address the heavy-tailed distribution of overestimation bias ($\ipssimple (\pi) - V(\pi)$) during HPO, as suggested in Figure~\ref{fig:overestimation_bias}. Proposition~\ref{prop:regret} implies that the overestimation of the value of hyperparameters sampled during HPO can exacerbate the generalization regret of an HPO algorithm. To deal with this issue, we introduce \textbfit{conservative surrogate objective}, which penalizes the validation performance of hyperparameters whose performance has a large uncertainty to avoid the issue of overestimation bias during HPO. Specifically, we propose to use a high probability lower bound of the generalization performance (denoted as $\hat{V}_{-}(\cdot)$) as an alternative surrogate objective, which is given as: $\mP \big( V (\pi) \ge \hat{V}_{-} (\pi; \calD, \delta)  \big) \ge 1 - \delta$ where $\delta \in (0,1)$ specifies a confidence level.

A prevalent strategy to construct $\hat{V}_{-}(\cdot)$ in OPE is to apply a concentration inequality such as Hoeffding and Bernstein~\cite{thomas2015confidence,thomas2015high}. A problem is that these inequalities are often overly conservative as they make no assumptions about underlying distribution. Thus, we use an alternative strategy to construct $\hat{V}_{-}(\cdot)$ based on the Student's t-distribution as follows.
\begin{align}
    \hat{V}^{t}_{-} (\pi; \calD, \delta) := \ipssimple(\pi; \calD) - t_{1-\delta, \nu} \sqrt{\frac{\mV_n (\ipssimple (\pi; \calD)) }{n-1}} \label{eq:t_lower_bound},
\end{align}
where $t_{1-\delta, \nu}$ is the T-value given confidence level $\delta$ and degrees of freedom $\nu$.

The upside of Eq.~\eqref{eq:t_lower_bound} is that it produces a tighter lower bound than aforementioned concentration inequalities. This is because Eq.~\eqref{eq:t_lower_bound} introduces the additional assumption that the mean of importance weighted rewards $(\pi/\pi_0)r $ is normally distributed. This assumption is reasonable with growing data sizes. However, $(\pi/\pi_0)r $ often follows a distribution with heavy upper tails, which may make the assumption invalid in a small sample setting. Nonetheless, Appendix E empirically verifies that Eq.~\eqref{eq:t_lower_bound} is \textit{reasonably tight} compared with other popular concentration inequalities.

\begin{algorithm}[t]
\caption{Conservative and Imitation-Regularized HPO}
\label{algo:ss_hpo_ht}
\begin{algorithmic}[1]
\REQUIRE $A$, $\delta$, $\gamma$, $\Theta$, $T$, $\pi_0$, $\calDtr$, $\calDval$
\ENSURE $\hat{\theta}$ 
\STATE $\calS_0 \leftarrow \emptyset$
\FOR{$t = 1,2, \ldots, T$ }
    \STATE $\theta_t \leftarrow A(\theta \,|\, \calS_{t-1}) $ \hfill \textit{\small // sample candidate hyperparameters}
    \STATE $\hat{\pi}_t \leftarrow \hat{\pi}(\cdot \,|\, \theta_t, \calDtr) $ \hfill \textit{\small // train a policy (lower-level)}
    \STATE $\pi_t \leftarrow (1-\alpha_t) \cdot \hat{\pi}_t + \alpha_t \cdot \pi_0 $ \hfill \textit{\small // regularization (Eq.~\eqref{eq:adaptive_regularization_param})}
    \IF{ $ \hat{V}^t_{-}(\pi_t; \calDval, \delta) \ge  \hat{V}^t_{-}(\pi^*; \calDval, \delta) $ }
        \STATE $\hat{\theta} \leftarrow \theta_t$, $\pi^* \leftarrow \pi_t$ \hfill \textit{\small // update the solution}
    \ENDIF
    \STATE $\calS_t \leftarrow \calS_{t-1} \cup \{(\theta_t, \hat{V}^t_{-}(\pi_t; \calDval, \delta))\}$ \textit{\small // store the result}
\ENDFOR
\end{algorithmic}
\end{algorithm}

\begin{figure*}[ht]
\includegraphics[width=16.5cm]{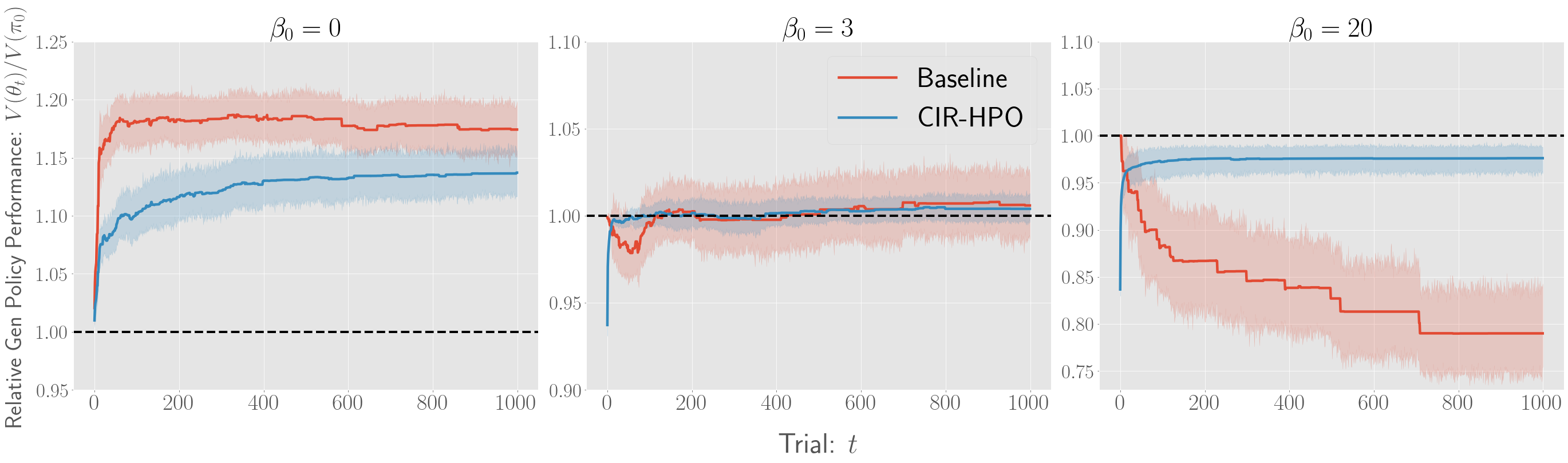}
\vspace{-2mm}
\caption{Comparing CIR-HPO (our proposal) and Baseline by their generalization performance. The results are averaged over 25 runs with different seeds and then normalized by $V(\pi_0)$. The shaded regions indicate 95\% confidence intervals.}
\label{fig:baseline_vs_chpo}
\end{figure*}

\begin{figure}[ht]
\centering
\includegraphics[width=8cm]{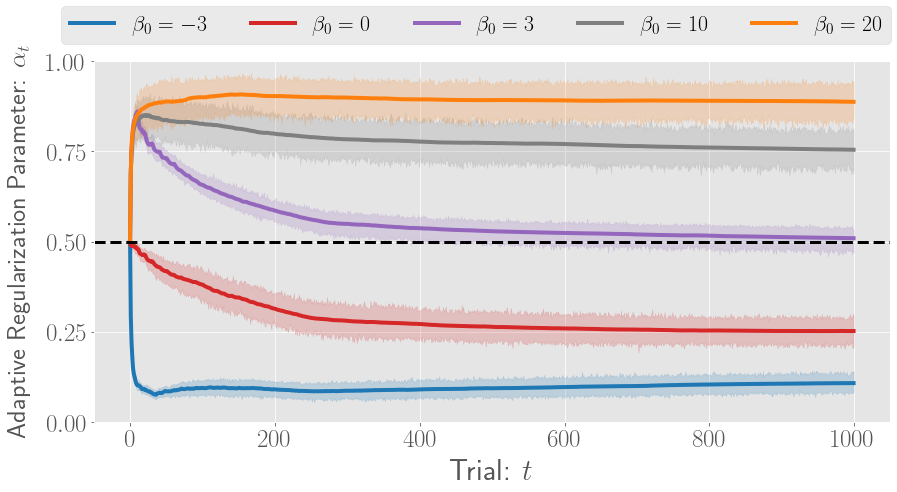}
\vspace{-2mm}
\caption{Behavior of adaptive regularization parameter ($\alpha_t$) of CIR-HPO with varying values of $\beta_0 \in \{-3,0,3,10,20\}$.}
\label{fig:air}
\end{figure}
\subsection{Adaptive Imitation Regularization (AIR)}
The second technique we propose is \textbfit{adaptive imitation regularization}, which tackles the unsafe behavior of the typical procedure. The issue of unsafe behavior suggests that, if logging policy $\pi_0$ is better than uniform random or is near-optimal, Algorithm~\ref{algo:naive_hpo} can produce a solution whose performance is much worse than that of the logging policy. 
Avoiding this problem is non-trivial, because we do not have access to the generalization performance and do not know the optimality of the logging policy in practice. For example, simply setting $\pi_0$ as an initial solution does not solve the issue at all, as suggested in Section~\ref{sec:empirical_analysis}. An instant idea might be to \textit{imitate} the logging policy to some extent:
\begin{align}
    \pi_t (a | x; \alpha, \theta_t, \calDtr) = (1-\alpha) \hat{\pi}(a | x; \theta_t, \calDtr) + \alpha \pi_0 (a | x),
    \label{eq:static_regularization}
\end{align}
where $\theta_t$ is a set of hyperparameters sampled at the $t$-th trial. $\alpha \, (\in [0,1])$ is a regularization parameter, which mixes the policy induced by $\theta_t$ and $\pi_0$ to construct a policy to evaluate. A large value of $\alpha$ makes $\pi_t$ closer to the logging policy, possibly avoiding the unsafe behavior. However, if the logging policy is detrimental, we should use a small $\alpha$ so that we can avoid an unnecessary performance degradation. So, a natural question to ask here is: \textit{how should we set the regularization parameter $\alpha$?} Again, this problem is non-trivial, as the optimality of the logging policy is unknown when performing HPO.

To overcome this difficulty in correctly setting $\alpha$, we propose \textit{adaptively tuning this parameter over the course of HPO}.
Based on the previous discussion, we should apply a strong regularization if $\pi_0$ performs well, otherwise we should not imitate $\pi_0$.
A key idea here is that we can reason about the optimality of the logging policy by comparing it with solutions sampled during HPO, i.e., $\{ \hat{\pi}(a | x; \theta_t, \calDtr) \}_{t=1}^T$. If most of the sampled solutions underperform $\pi_0$, we can infer that the logging policy is well-performing. To make a valid comparison between the sampled solutions and the logging policy, we apply a Student’s t-test based on the following T-value.
\begin{align*}
    T(\pi_1, \pi_2) := \frac{|\ipsdiff|}{\sqrt{\hat{\mV}_n (\ipsdiff)/(n-1)}}.
\end{align*}
where $\ipsdiff:= \ipssimple(\pi_1) - \ipssimple(\pi_2)$ is the performance difference between the two policies estimated by IPS.
Given a null hypothesis ($\ipsdiff=0$) and a normality assumption, $T(\pi_1, \pi_2)$ follows a t-distribution with $\nu$ degrees of freedom. 
We then calculate the optimality score of $\pi_0$ at the $t$-th trial as follows.
\begin{align}
    s_t = 
        \begin{cases}
        1 & (T(\pi_0, \pi_t) \ge t_{1-\delta/2, \nu}  \textit{ and } \Delta \ipssimple (\pi_0, \pi_t) \ge 0)\\
        - 1 & (T(\pi_0, \pi_t) \ge t_{1-\delta/2, \nu}  \textit{ and }  \Delta \ipssimple (\pi_0, \pi_t) < 0)\\
        0 & (\textit{otherwise, i.e., } T(\pi_0, \pi_t) < t_{1-\delta/2, \nu})
        \end{cases}
\end{align}
$s_t$ indicates whether $\pi_0$ is better or worse than $\pi_t$ in a significant level.
If $\pi_0$ is better than $\pi_t$, then $s_t = 1$. Instead, $s_t=-1$ if $\pi_0$ is tested to be worse.
If there is no significant difference between $\pi_0$ and $\pi_t$, the score is zero.

Using the sequence of scores up to the $t$-th trial, i.e., $\{ s_{t'} \}_{t'=1}^t $, we define \textit{adaptive regularization parameter} as:
\begin{align}
    \alpha_t 
    &:= \alpha_{init} + (1-\alpha_{init}) \cdot  \left( \frac{t}{T} \right)^{\gamma} \cdot \frac{\sum_{t'=1}^t s_{t'}}{t}
    \label{eq:adaptive_regularization_param}
\end{align}
where $\alpha_{init} \in [0,1]$ is an initial regularization parameter and $\gamma \, (> 0)$ is a scheduling parameter for adaptive regularization. 
For example, suppose that $s_t=1, \forall t =1,2,\ldots,T$, meaning that $\pi_0$ is always better than $\pi_t$ in a significant level.
Then, following Eq.~\eqref{eq:adaptive_regularization_param}, $\alpha_T = 1$ and the HPO procedure outputs $\pi_0$, because it should be near-optimal.
On the other hand, if $s_t=-1, \forall t =1,2,\ldots,T$, meaning that $\pi_0$ is always worse than $\pi_t$ in a significant level, then $\alpha_T = 0$ and the HPO procedure does not imitate the logging policy at all, because it should be a bad policy.

\subsection{The CIR-HPO Algorithm}
Algorithm~\ref{algo:ss_hpo_ht} describes the CIR-HPO algorithm, which leverages conservative surrogate objective (lines 6 and 9) and adaptive imitation regularization (line 5). $\delta$ and $\gamma$ are \textit{meta hyperparameters}. $\delta$ controls how conservative we would like to be during HPO, and $\gamma$ controls the scheduling of the adaptive regularization. In Section~\ref{sec:experiment}, we show that these configurations have some impact on the behavior of CIR-HPO, but we also demonstrate that the default values ($\delta=0.1$ and $\gamma=0.01$) work reasonably well in a range of experiment settings. The other inputs are the same as those of Algorithm~\ref{algo:naive_hpo}. Note that our algorithm is easy to implement with a few additional lines of code and there is no additional computational overhead compared to the typical procedure in Algorithm~\ref{algo:naive_hpo}.

\section{Empirical Evaluation}\label{sec:experiment}
This section empirically compares \textbf{Baseline} (Algorithm~\ref{algo:naive_hpo}) and \textbf{CIR-HPO} (Algorithm~\ref{algo:ss_hpo_ht}), employing the same synthetic data and policy class as in Section~\ref{sec:empirical_analysis}.
Note that we compare CIR-HPO against only \textbf{Baseline} because there is no other method proposed for HPO using logged bandit data (comprehensive summary of related work can be found in Appendix B).

\subsection{Baseline vs CIR-HPO}
Figure~\ref{fig:baseline_vs_chpo} compares the performance of \textbf{Baseline} and \textbf{CIR-HPO} with varying logging policies ($\beta_0 \in \{0,3,20\}$). First, when $\beta_0=0$ where the logging policy is uniform random, both Baseline and CIR-HPO work reasonably well and succeed in finding a set of hyperparameters that leads to a policy much better than the logging policy. What is notable for this setting is that CIR-HPO is inefficient and slow to converge compared to Baseline due to adaptive imitation regularization, even though it reaches far above the black horizontal line ($V(\pi_0)$). At the initial stage of HPO, we do not know how close the logging policy is to the optimal policy. Therefore, the proposed procedure gradually learns the optimality of the logging policy, potentially leading to a slower convergence if the logging policy is far from optimal (such as uniform random). Next, when $\beta_0=3$ where the logging policy is better than uniform random, but is not close to the optimal, both Baseline and CIR-HPO slightly improve the logging policy. However, the confidence intervals indicate that CIR-HPO is much more stable than Baseline. In particular, Baseline is much more likely to \textit{underperform} the logging policy, even though it outperforms the logging policy on average. Finally, when $\beta_0=20$ where the logging policy is near-optimal, Baseline outputs a solution that is substantially worse than the logging policy, even though it starts from the logging policy as its initial solution. In contrast, CIR-HPO learns that the logging policy is near-optimal during HPO and strengthens the imitation regularization adaptively. As a result, it prevents the solution from being significantly worse than the (already near-optimal) logging policy, which is compelling, because we do not know the optimality of the logging policy in advance. Figure~\ref{fig:air} illustrates the behavior of adaptive imitation regularization, which suggests that it succeeds in controlling the strength of regularization depending on the optimality of the logging policy.

\begin{figure}[th]
\centering
\includegraphics[width=8.5cm]{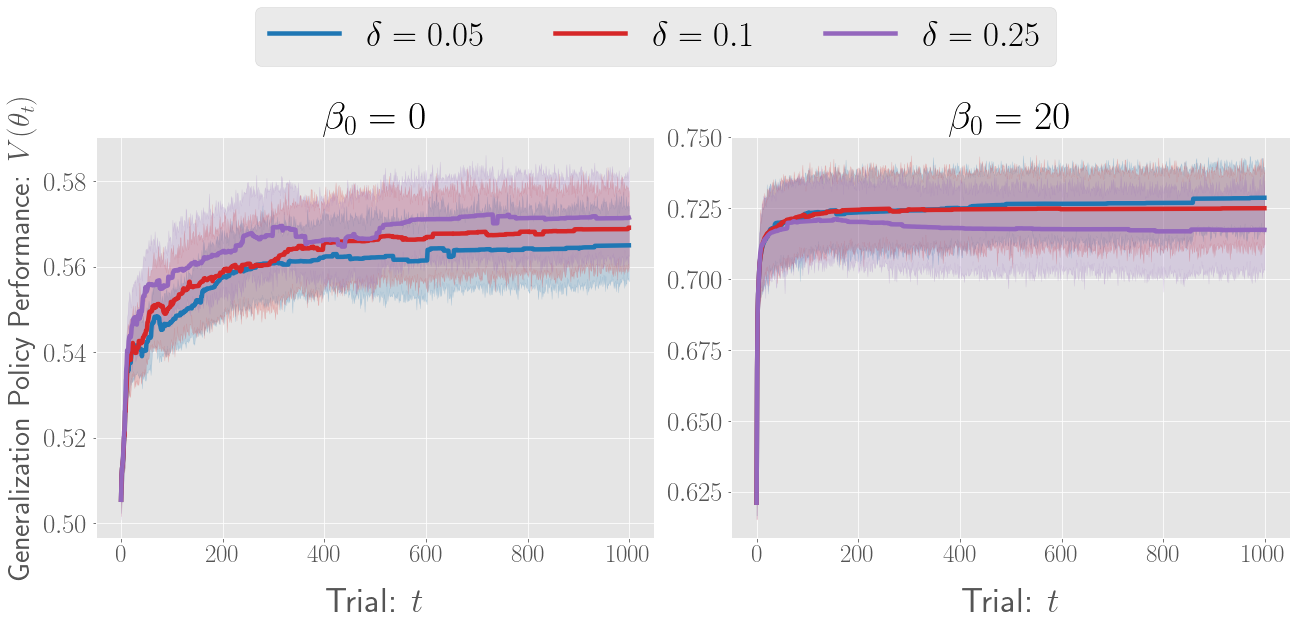}
\vspace{-2mm}
\caption{Sensitivity of the generalization performance of CIR-HPO regarding the choice of $\delta$.}
\label{fig:varying_delta}
\end{figure}

\begin{figure}[th]
\centering
\includegraphics[width=8.5cm]{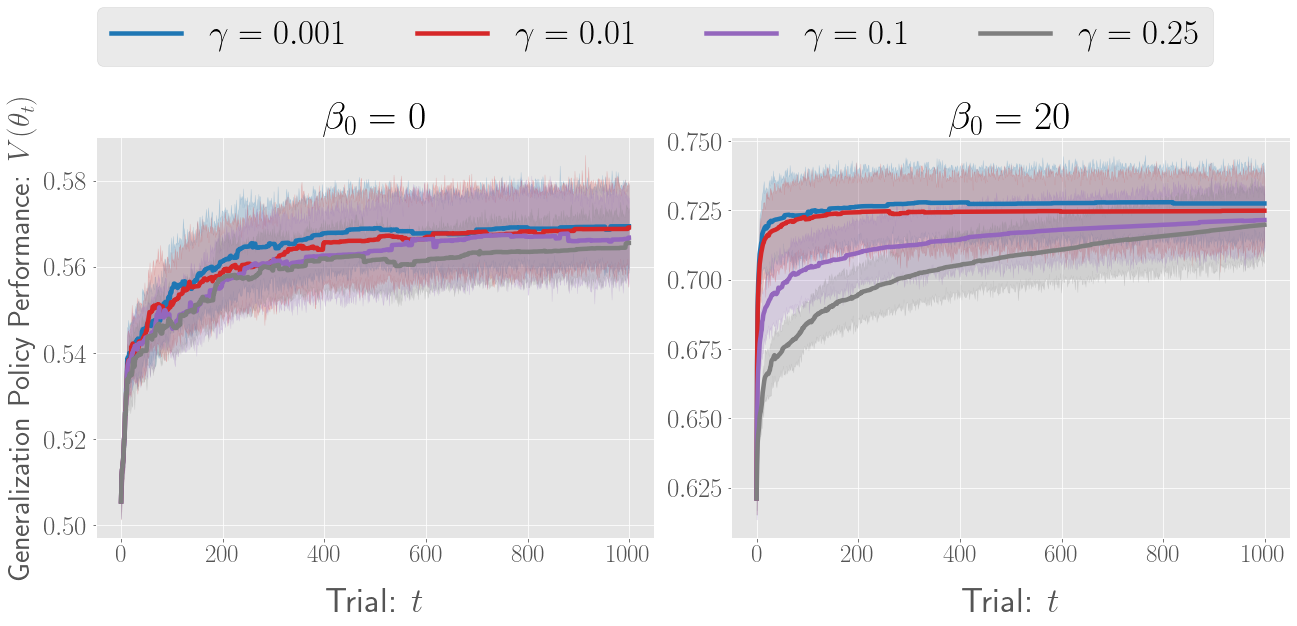}
\vspace{-2mm}
\caption{Sensitivity of the generalization performance of CIR-HPO regarding the choice of $\gamma$.}
\label{fig:varying_gamma}
\end{figure}

\begin{figure}[th]
\centering
\includegraphics[width=8cm]{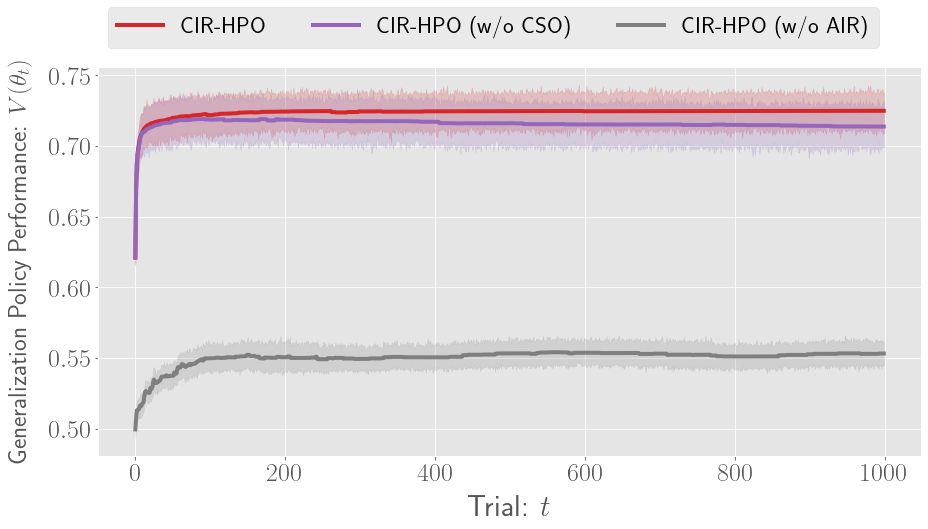}
\vspace{-2mm}
\caption{Ablation study of CIR-HPO ($\beta_0 = 20$).}
\label{fig:ablation}
\end{figure}
\subsection{Choice of Meta Hyperparameters}
Next, we evaluate the sensitivity of CIR-HPO to the choice of its meta hyperparameters. Figure~\ref{fig:varying_delta} shows that the effectiveness of CIR-HPO with different values of $\delta$. The result demonstrates that there is no significant difference among the three values, suggesting that we do not have to care too much about which value to use for $\delta$. In addition, Figure~\ref{fig:varying_gamma} evaluates different values of $\gamma$, which controls the scheduling of adaptive imitation regularization. This result implies that, for a sub-optimal logging policy ($\beta_0=0$), the choice of $\gamma$ has no significant effect on the behavior of CIR-HPO. For a near-optimal logging policy ($\beta_0=20$), however, a smaller $\gamma$ leads to a faster convergence, although all values achieve the same level of performance in the final stage.

\subsection{Ablation Study}
We also conduct an ablation study to evaluate the contribution of \textbf{conservative surrogate objective (CSO)} and \textbf{adaptive imitation regularization (AIR)} to the effectiveness of CIR-HPO. To this end, we compare \textbf{CIR-HPO} to \textbf{CIR-HPO (w/o CSO)} and \textbf{CIR-HPO (w/o AIR)} in Figure~\ref{fig:ablation}. The result demonstrates that both CSO and AIR clearly contribute to the performance of CIR-HPO, while AIR has a more appealing effect (CSO and AIR provide 1.6\% and 23.6\% improvements, respectively, in terms of the final generalization performance).

\subsection{A Real-World Experiment}
In addition to the synthetic experiments, we apply CIR-HPO to the Open Bandit Dataset~\cite{saito2021open}, a publicly available logged bandit dataset collected on a large-scale fashion e-commerce platform. The results suggest that CIR-HPO leads to a better policy compared to the Baseline procedure in terms of the generalization performance, providing a further argument regarding its real-world applicability. The experiment detail and results can be found in Appendix A.

\section{Conclusion}
This work studies the effectiveness of the typical HPO procedure in the OPL setup from both empirical and theoretical perspectives and found that it can fail and even be harmful. In particular, we investigated two surprising issues, namely optimistic bias and unsafe behavior, and showed that a heavy-tailed distribution of overestimation can cause an unexpected gap between validation and generalization. In response, we made two extremely simple corrections to the typical HPO procedure, resulting in the CIR-HPO algorithm, to deal with the issues. Extensive experiments demonstrated that CIR-HPO can be advantageous, particularly when the conventional procedure collapses and causes a significant and undetectable deterioration in the generalization performance.

\bibliographystyle{named}
\bibliography{main.bbl}

\clearpage
\clearpage
\appendix

\begin{table}[t]
\caption{Test policy values of the policies tuned by Baseline and CIR-HPO in the real-world experiment.} \label{tab:real}
\centering
\scalebox{1.1}{\begin{tabular}{c|cc}
\hline
  & IPS & DR \\ \hline
Baseline & 5.199 $\times 10^{-3}$ & 5.128 $\times 10^{-3}$ \\ 
CIR-HPO (ours) & \textbf{5.214 $\times 10^{-3}$} & \textbf{5.303 $\times 10^{-3}$}  \\ 
\hline
\end{tabular}}
\end{table}

\section{A Real-World Experiment}

\subsection{Dataset.}

To assess the real-world applicability of our CIR-HPO, here we evaluate it on the Open Bandit Dataset (OBD)\footnote{https://research.zozo.com/data.html}~\cite{saito2021open}, a publicly available logged bandit dataset collected on a large-scale fashion e-commerce platform. We use 100,000 observations that are randomly sub-sampled from the ``Men's" campaign data of OBD. The dataset contains user contexts $x$, fashion items to recommend as action $a \in \calA$ where $|\calA|=240$, and resulting clicks as reward $r \in \{0,1\}$.

The dataset consists of two sets of logged bandit data collected by two different policies (uniform random and Thompson sampling) during an A/B test of these policies. We regard Thompson sampling as a logging policy and perform HPO of a policy class defined in Section~\ref{sec:empirical_analysis}. We then approximate the ground-truth performance of the tuned policies on the test dataset collected by uniform random. Note that we use the same policy class defined in Section~\ref{sec:empirical_analysis} and the default meta-parameters of CIR-HPO described in Section~\ref{sec:experiment}. All the experiments were conducted on MacBook Pro (2.4 GHz Intel Core i9, 64 GB).

\subsection{Result.}

Table~\ref{tab:real} reports the results (averaged over 5 runs with different random seeds) of the real-world experiment. We compare \textbf{Baseline} and \textbf{CIR-HPO (ours)} combined with IPS and DR as OPE estimators to provide a surrogate objective (i.e., $\hat{V}(\theta; \calD)$). The results suggest that, for both estimators, CIR-HPO outperforms Baseline in terms of the test policy value. This observation provides further arguments for the applicability of our CIR-HPO.

\section{Related Work} \label{app:related}

\paragraph{Off-Policy Evaluation and Learning.}
The basis of our study lies in OPE, which is interested in accurately estimating the generalization policy performance from logged bandit data.
This has been one of the most fundamental problems in contextual bandits and RL, with applications ranging from recommender systems~\cite{saito2021counterfactual,mcinerney2020counterfactual,kiyohara2022doubly,saito2022off,saito2020unbiased,saito2020pairwise,saito2022towards,saito2021evaluating} to personalized medicine~\cite{tang2021model,kallus2021optimal,saito2022counterfactual}.
The most common solution in OPE is to use IPS weighting. IPS provides an unbiased estimate of the policy performance. However, there is a canonical criticism that IPS often suffers from a high variance due to a low overlap~\cite{dudik2014doubly,wang2016optimal}.
Thus, alternative estimators have been explored to reduce the variance without introducing large bias. 
For example, Self-Normalized IPS (SNIPS)~\cite{swaminathan2015self} aims to reduce the variance of IPS as follows.
\begin{align*}
    \snips := \frac{1}{\sum_{i=1}^n \frac{\pi_e(a_i | x_i)}{\pi_0 (a_i | x_i)}} \sum_{i=1}^n  \frac{\pi_e(a_i | x_i)}{\pi_0 (a_i | x_i)}  r_i.
\end{align*}
This estimator normalizes the IPS estimate by the sum of the importance weights ($\sum_{i=1}^n \frac{\pi_e(a_i | x_i)}{\pi_0 (a_i | x_i)}$) to gain stability.
Moving forward, DR leverages a control variate to provide an \textit{efficient} OPE. The DR estimator is defined as follows.
\begin{align*}
    \dr := \frac{1}{n} \sum_{i=1}^n \hat{\mu}(x_i, \pi_e) + \frac{\pi_e(a_i | x_i)}{\pi_0 (a_i | x_i)}  (r_i-\hat{\mu}(x_i, a_i) ),
\end{align*}
where $\hat{\mu}(x,\pi_e)$ estimates $\mE_{\pi_e}[\mu(x,a)]$. This estimator is still unbiased and consistent if either the importance weight or the reward estimator is true or consistent. In addition, DR is \textit{efficient} in that it reaches the lowest achievable asymptotic variance if the reward estimator is correctly specified. There have also been much efforts to further improve DR in a finite sample setting such as Switch~\cite{wang2016optimal}, More Robust Doubly Robust~\cite{farajtabar2018more}, and Shrinkage~\cite{su2020doubly}.

Instead, OPL is the task of improving the decision making policies using only logged bandit data collected from a logging policy~\cite{swaminathan2015batch}. The optimal policy maximizes the generalization performance, i.e., $\pi^* := \argmax_{\pi} V(\pi)$. However, this problem is intractable because we cannot know the generalization performance. This raises the need for applying an estimator for its careful approximation, as done in the empirical risk minimization of supervised ML. A typical estimator choice for OPL is IPS~\cite{swaminathan2015batch,ma2019imitation,joachims2018deep} or its variants~\cite{swaminathan2015self}. A problem is that the variance issue arises here again. Thus, research has been centered around adding regularization to deal with the variance issue during policy training. The fundamental method is variance regularization, which penalizes the policy whose variance in the performance estimation is high~\cite{swaminathan2015batch}. Other regularization methods include imitation regularization~\cite{ma2019imitation} and behavior regularization~\cite{wu2019behavior,kumar2019stabilizing}. \cite{jeunen2021pessimistic} explore the optimistic bias in OPL, and propose a pessimistic reward modeling for OPL based on a Bayesian uncertainty estimation. Instead, we focus on investigating and alleviating the optimistic bias in HPO and empirically illustrate the unsafe behavior, which is specific to our HPO setup.

\paragraph{Off-Policy Selection.}
Off-Policy Selection (OPS) is a sub-field of OPE and OPL and is closely related to our HPO setting. This is the task of identifying the best policy out of a given \textit{finite} set of candidate policies using only logged bandit data. We can view this selection problem as a special case of OPL, where the policy class $\Pi$ is finite. \cite{kuzborskij2021confident} study OPS in the contextual bandit setting. They develop a confident OPS procedure, which is based on an Efron-Stein high probability lower bound of the policy performance derived from SNIPS. \cite{yang2020offline} study OPS in RL and propose BayesDICE for estimating the brief over the performance of the candidate policies, which is useful for the selection task. \cite{doroudi2018importance} theoretically characterize a failure of IPS in OPS. Specifically, \cite{doroudi2018importance} show that naively applying IPS to OPS can result in an \textit{unfair} selection in the sense that the procedure can select the worst of the two candidate policies more often than not.
\cite{paine2020hyperparameter} and \cite{fu2020benchmarks} conduct empirical studies on OPS of RL polices for several benchmark control tasks. They identify Fitted Q Evaluation as a useful strategy for OPS in RL.
\cite{tang2021model} also provide an empirical study on OPS of RL policies and propose to combine multiple OPE estimators for an accurate and scalable OPS.

Although these studies on OPS are closely related, our contributions are unique in several ways. First, we focus on HPO, not OPS, which adaptively finds better hyperparameter solutions given a certain budget. By paying attention to HPO, our empirical analysis succeeded in finding that the TPE algorithm, a popular adaptive method in HPO, cannot improve the generalization performance of OPL. This is our unique finding, not captured by previous studies targeting only OPS. Second, we provide a theoretical analysis about the optimistic bias and the gap in generalization and validation regret, explaining the empirical observations. Although \cite{paine2020hyperparameter} point out the overestimation bias in the context of OPS, they provide no theoretical explanation. Finally, we propose CIR-HPO based on our empirical observations and analysis. This procedure is specific to the adaptive optimization process and is non-trivial given any existing studies on OPS.

\paragraph{Hyperparameter Optimization (HPO).}
HPO is a critical element for the success of a range of machine learning algorithms and tasks~\cite{feurer2019hyperparameter}.
For instance, hyperparameter configurations can entirely change the performance of deep neural networks~\cite{dacrema2019we,henderson2018deep,lucic2018gans}.
A typical formulation regards HPO as a black-box optimization problem, where the input is a set of hyperparameters, and the output is a validation performance (an accessible proxy of the generalization performance). 
Among many black-box optimization methods, Bayesian optimization (BO)~\cite{brochu2010tutorial,shahriari2015taking,frazier2018tutorial}, such as Gaussian process bandit algorithms~\cite{srinivas2010gaussian} and tree-structured Parzen estimator (TPE)~\cite{bergstra2011algorithms} have gained particular attention.
These methods sequentially optimize the hyperparameters of a prediction model by leveraging the previous evaluation results to sample the next set of hyperparameters to evaluate.
More specifically, previous evaluation results are used to train a surrogate to model the relationship between hyperparameters and the resulting prediction accuracy. Then, the algorithms balance the exploration and exploitation based on an acquisition function, such as expected improvement and upper confidence bound. Because of the sample efficiency, BO demonstrates a state-of-the-art performance with a limited budget~~\cite{turner2021bayesian}.
It should be noted that, while this study focuses on BO, our discussion can be applied to other optimization methods such as CMA-ES~\cite{hansen2016cma,nomura2024cmaes}, whose efficiency is verified in multiple HPO tasks~\cite{loshchilov2016cma,nomura2021warm}.

A critical convention in HPO research is to evaluate the performance and efficiency of algorithms based solely on the validation performance.
This implies that there is an implicit and often neglected assumption that optimizing the validation performance is a reasonable strategy for optimizing the generalization performance (primary objective). However, it is unclear whether optimizing the validation performance really improves the generalization performance. In fact, Section~\ref{sec:analysis} sheds light on the fact that ignoring this assumption in OPL can lead to an unexpected failure and a substantial validation-generalization gap. Our theoretical and empirical illustrations might also contribute to a broader HPO community, as there are few studies verifying whether naively setting the validation performance as a surrogate objective is reasonable, given the goal of optimizing the generalization performance.

\section{Omitted Proofs} \label{app:proof}
This section provides proofs omitted in the main text.

\subsection{Proof of Proposition~\ref{prop:optimistic_bias}}

\begin{proof}
Given that $\theta^*$ and $\hat{\theta}$ are defined in Eq.~\eqref{eq:opt_hypara} and Eq.~\eqref{eq:hpo}, we have that 
\begin{align*}
    & \mE \left[\hat{V} (\hat{\pi}(\cdot \mid \cdot, \hat{\theta}, \calD_{tr}); \calD_{val} ) \right] \\
    &\ge \mE \left[ \hat{V} \left(\hat{\pi}(\cdot \mid \cdot, \theta^*, \calD_{tr}); \calD_{val} \right) \right] 
    = V \left(\hat{\pi}(\cdot \mid \cdot, \theta^*, \calD_{tr}) \right),
\end{align*}
where the last equation follows, as $\theta^*$ does not depend on $\calD_{val}$.
Similarly, the right inequality of Eq.~\eqref{eq:optimistic_bias} comes from the fact that $\theta^*$ is optimal in terms of the true generalization policy performance.
\end{proof}

\subsection{Proof of Proposition~\ref{prop:regret}}
\begin{proof}
Our derivation is inspired by the regret analysis provided in \cite{nomura2021efficient}.
Given the notations introduced in Section~\ref{sec:theoretical_analysis}, it follows that
\begin{align*}
&r_{gen}(T; A, \calD) \\
& = V \left(\theta^*\right) - V ( \hat{\theta}_{T,A} (\calD) ) \\
    &= \underbrace{V \left(\theta^*\right) - \ipssimple(\theta^*; \calD)}_{= - \tau(\theta^*; \calD)} + \ipssimple(\theta^*; \calD) - V ( \hat{\theta}_{T,A} (\calD) ) \\
    &= - \tau(\theta^*; \calD) + \underbrace{(- V(\hat{\theta}_{T,A} (\calD)) + \ipssimple(\hat{\theta}_{T,A} (\calD); \calD) )}_{= \tau(\hat{\theta}_{T,A}(\calD); \calD)} \\
    &\qquad - \ipssimple(\hat{\theta}_{T,A}(\calD); \calD) + \ipssimple(\theta^*; \calD) \\
    &= \Delta \tau(\hat{\theta}_{T,A}(\calD), \theta^*) + \underbrace{(\ipssimple(\hat{\theta}^*; \calD) - \ipssimple(\hat{\theta}_{T,A}(\calD); \calD))}_{= r_{val}(T; A, \calD)} \\
    &\qquad     
    + \underbrace{(\ipssimple(\theta^*; \calD) - \ipssimple(\hat{\theta}^*; \calD))}_{= C} \\
    &= r_{val}(T; A, \calD) + \Delta \tau(\hat{\theta}_{T,A}(\calD), \theta^*; \calD) + C.
\end{align*}
\end{proof}

\section{Additional Theoretical Result}
We suppose $\hat{V}(\theta)$ has the following form:
\begin{align*}
    \hat{V}(\theta) = \frac{1}{n}\sum_{i=1}^n v(a_i, x_i; \theta).
\end{align*}
Note that this form is general and encompasses common estimators.
For example, we can obtain IPS estimator by setting $v(a_i, x_i; \theta) = \frac{1}{n}\sum_{i=1}^n \frac{\pi_e(a_i | x_i; \theta)}{\pi_0 (a_i | x_i)} r_i$.
The following inequality suggests that the optimistic bias $\hat{V} \big(\hat{\theta}(\calD); \calD \big)
     - V \left( \theta^* \right)$ decreases at the order $\mathcal{O}\left( 1 / \sqrt{n} \right)$ as the data increases.

\begin{proposition}
    Suppose $|\Theta| < \infty$ and $0 \leq v(a,x,\theta) \leq 1$ for all $a \in \calA, x \in \calX, \theta \in \Theta$. For $\delta \in (0, 1)$, the following inequality holds with probability as least $1 - \delta$:
    \begin{align*}
        \hat{V} \big(\hat{\theta}(\calD); \calD \big) 
     - V \left( \theta^* \right) \leq \sqrt{\frac{1}{2n}\log\frac{|\Theta|}{\delta}} \in \mathcal{O}\left( \frac{1}{\sqrt{n}} \right).
    \end{align*}
\end{proposition}
\begin{proof}
We first decompose the optimistic bias as
\begin{align*}
    &\hat{V} \big(\hat{\theta}(\calD); \calD \big) 
     - V \left( \theta^* \right) \\
     &= \hat{V} \big(\hat{\theta}(\calD)
     ; \calD \big) - V(\hat{\theta}(\calD)) + \underbrace{V(\hat{\theta}(\calD))
     - V \left( \theta^* \right)}_{\leq 0} \\
     &\leq \hat{V} \big(\hat{\theta}(\calD); \calD \big) - V(\hat{\theta}(\calD)).
\end{align*}

Hence, for $\epsilon > 0$,
\begin{align*}
&\mP \left( \hat{V} \big(\hat{\theta}(\calD); \calD \big) 
     - V \left( \theta^* \right) \geq \epsilon \right) \\
&\leq \mP \left( \hat{V} \big(\hat{\theta}(\calD); \calD \big) - V(\hat{\theta}(\calD)) \geq \epsilon \right) \\
&\leq \mP \left( \bigcup_{\theta \in \Theta} \left\{ \hat{V} \big(\theta(\calD); \calD \big) - V(\theta(\calD)) \geq \epsilon \right\} \right) \\
&\leq \sum_{\theta \in \Theta} \mP \left( \hat{V} \big(\theta(\calD); \calD \big) - V(\theta(\calD)) \geq \epsilon \right) \\
&\leq |\Theta| e^{-2n\epsilon^2}.
\end{align*}
We used the union bound and Hoeffding's inequality\footnote{By replacing Hoeffding's inequality with Chebyshev's inequality, we can obtain a weaker result even if the boundedness of $v(a,x,\theta)$ is not assumed.}.
Putting the RHS as $\delta$ and solving it for $\epsilon$ completes the proof.
\end{proof}

\begin{figure*}[ht]
\centering
\includegraphics[width=16.5cm]{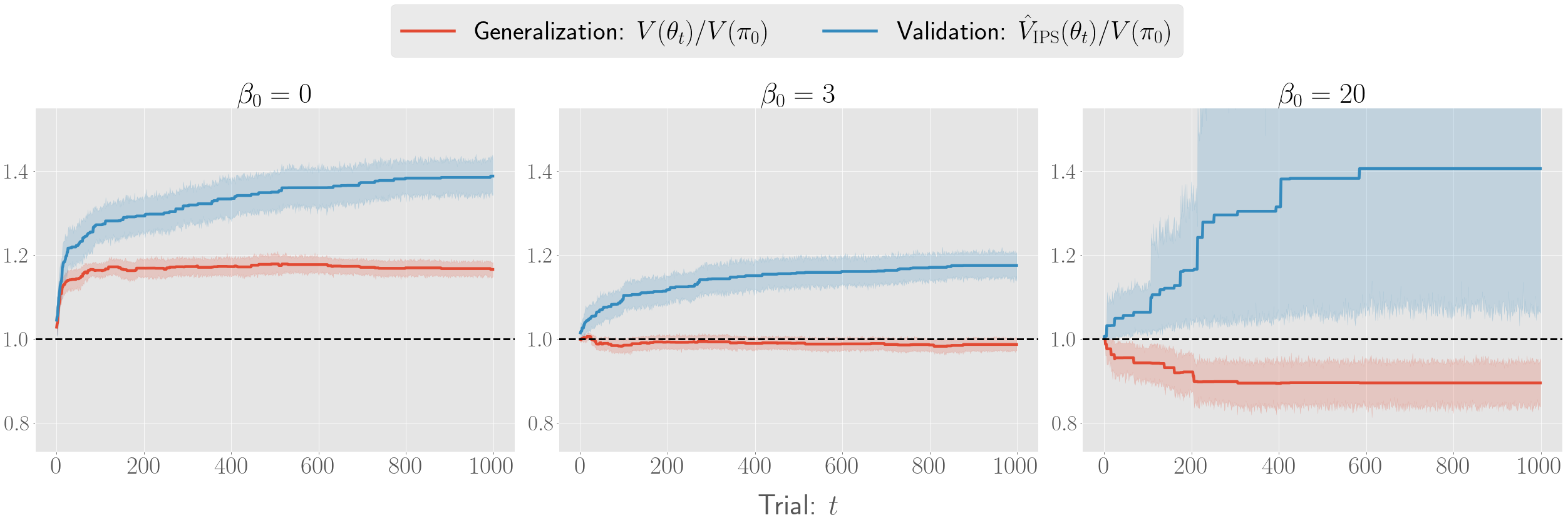}
\caption{Empirical Evidence of Optimistic Bias and Unsafe Behavior in HPO for OPL (w/ Random Search and the IPS estimator). The results are averaged over 25 runs with different seeds and then normalized by $V(\pi_0)$. The shaded regions indicate 95\% confidence intervals.}
\label{fig:optimistic_bias_random}
\end{figure*}

\begin{figure*}[ht]
\centering
\includegraphics[width=16.5cm]{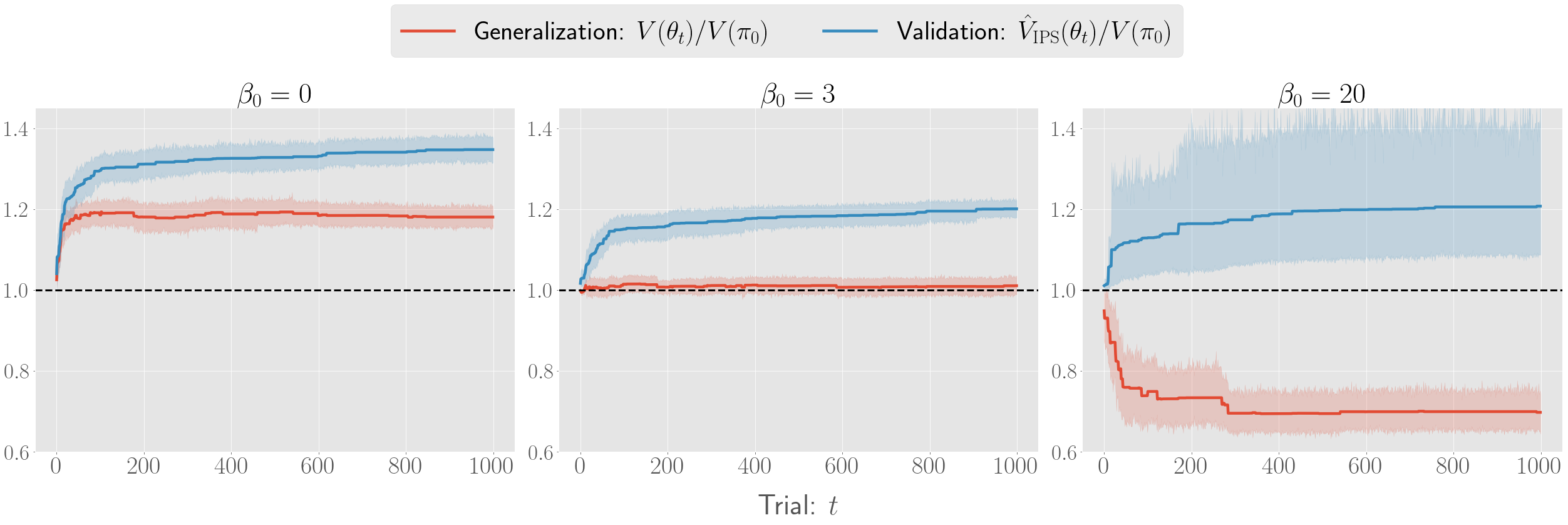}
\caption{Empirical Evidence of Optimistic Bias and Unsafe Behavior in HPO for OPL (w/ TPE and the DR estimator). The results are averaged over 25 runs with different seeds and then normalized by $V(\pi_0)$. The shaded regions indicate 95\% confidence intervals.}
\label{fig:optimistic_bias_dr}
\end{figure*}

\begin{figure*}[ht]
\centering
\begin{minipage}{1.0\textwidth}
    \begin{center}
        \begin{tabular}{c}
            \begin{minipage}{0.32\hsize}
                \begin{center}
                    \includegraphics[clip, width=6cm]{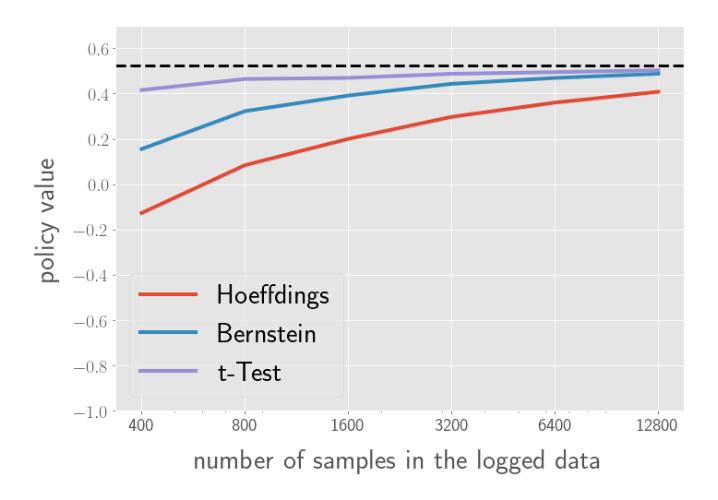}\\
                   $\beta_0 = 0$
                \end{center}
            \end{minipage}
            
            \begin{minipage}{0.32\hsize}
                \begin{center}
                    \includegraphics[clip, width=6cm]{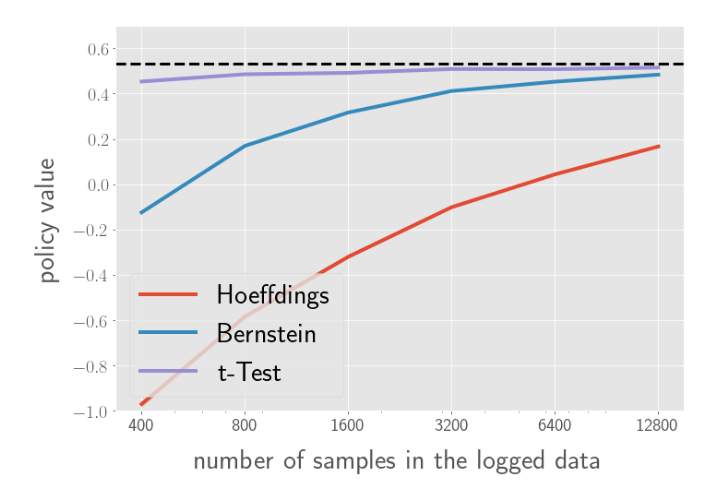}\\
                    $\beta_0 = 3$
                \end{center}
            \end{minipage}
            
            \begin{minipage}{0.32\hsize}
                \begin{center}
                    \includegraphics[clip, width=6cm]{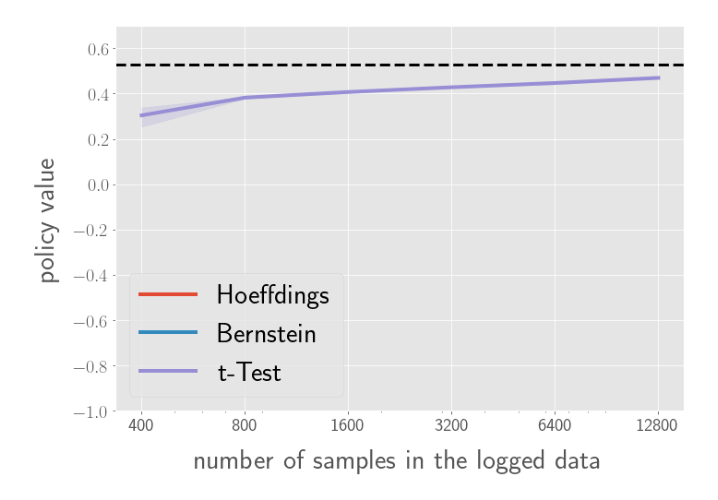}\\
                    $\beta_0 = 20$
                \end{center}
            \end{minipage}
        \end{tabular}
    \end{center}
    \caption{High Probability Lower Bound}
\label{fig:lower_bound}
\end{minipage}
\end{figure*}

\begin{figure*}[ht]
\centering
\begin{minipage}{1.0\textwidth}
    \begin{center}
        \begin{tabular}{c}
            \begin{minipage}{0.32\hsize}
                \begin{center}
                    \includegraphics[clip, width=6cm]{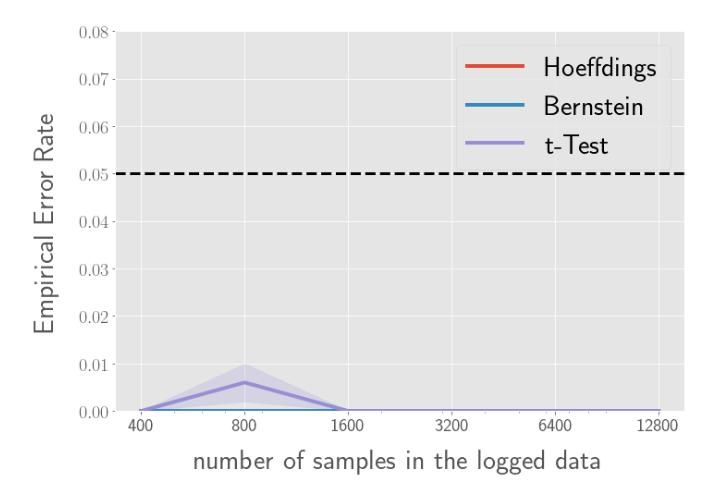}\\
                    $\beta_0 = 0$
                \end{center}
            \end{minipage}
            
            \begin{minipage}{0.32\hsize}
                \begin{center}
                    \includegraphics[clip, width=6cm]{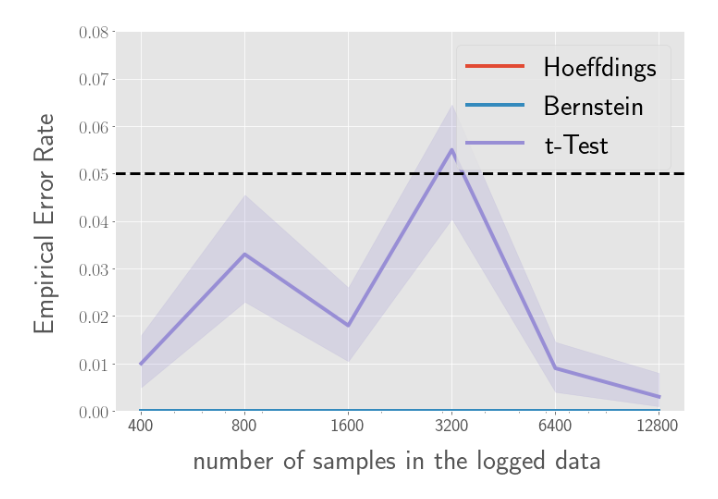}\\
                    $\beta_0 = 3$
                \end{center}
            \end{minipage}
            
            \begin{minipage}{0.32\hsize}
                \begin{center}
                    \includegraphics[clip, width=6cm]{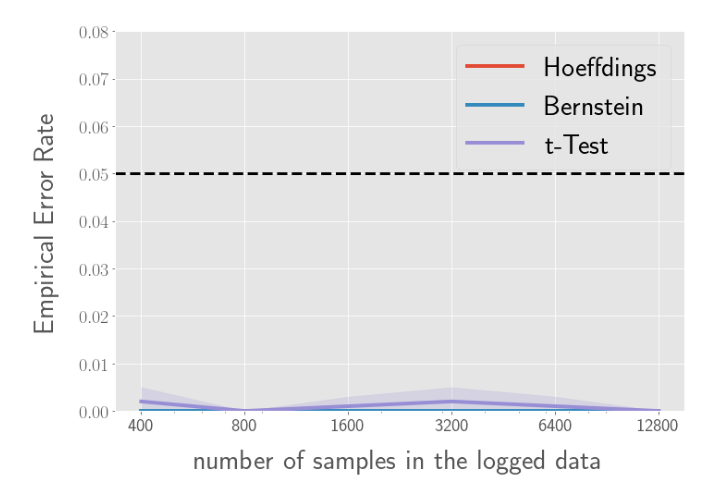}\\
                    $\beta_0 = 20$
                \end{center}
            \end{minipage}
        \end{tabular}
    \end{center}
    \caption{Empirical Error Rate}
\label{fig:error_rate}
\end{minipage}
\end{figure*}

\section{Supplemental Simulations} \label{app:simulation}
This section empirically evaluates the confidence lower bound of OPE based on concentration inequalities (Hoeffding and Bernstein) and a Student’s t-test. We follow Section~\ref{sec:analysis} to generate synthetic bandit data. We vary the value of $\beta_0$ within the range of $\{0,3,20\}$, and the number of validation data $|\calD_{val}|$ within the range of $\{400,800,1600,3200,6400,12800\}$. We also follow Section~\ref{sec:analysis} to train the evaluation policy $\pi_e$. Specifically, we first train $\hat{\mu}$ using logistic regression and form a stochastic policy based on Eq.~\eqref{eq:synthetic_eval} with $\beta=10$.

We use IPS and estimate a high probability lower bound of $V(\pi_e)$ based on Hoeffding, Bernstein, and t-Test. Given a confidence level $\delta \in (0,1)$, the estimated lower bounds are given as $\ipssimple (\pi_e; \calD_{val}) - f(\delta, \calD_{val})$ where 
\begin{align*}
    \textbf{Hoeffding}: &f(\delta, \calD_{val}) = w_{\max} \sqrt{\frac{2 \log(2/\delta)}{n}}, \\ 
    \textbf{Bernstein}: &f(\delta, \calD_{val}) = \sqrt{\frac{2 \log(2/\delta)\hat{\mV}(\ipssimple (\pi_e; \calD_{val}))}{n-1}}\\
    &\qquad \qquad \qquad + \frac{7w_{\max}\log(2/\delta)}{3(n-1)},\\
    \textbf{t-Test}: &f(\delta, \calD_{val}) = t_{1-\delta,\nu} \sqrt{\frac{\hat{\mV}(\ipssimple (\pi_e; \calD_{val}))}{n-1}}.
\end{align*}
Note that $n=|\calD_{val}|$ and $w_{\max}:= \sup_{(x,a) \in \calX \times \calA } \pi_e(a|x)/\pi_b(a|x)$. $t_{1-\delta,\nu}$ is the $100(1-\delta)$ percentile of the Student’s t distribution with $\nu$ degrees of freedom. The lower bound given by the t-test is based on the assumption that $(\pi_e(a|x)/\pi_0(a|x))r$ is normally distributed.

Figure~\ref{fig:lower_bound} shows the estimated lower bounds with varying $\beta$ and sample size, and with a fixed $\delta=0.05$.
The black horizontal line represents the ground-truth policy value $V(\pi_e)$. We observe that the lower bound given by t-Test is the tightest, while those by Hoeffding and Bernstein are invisible when $\beta_0=20$, as they are too loose. Bernstein is always tighter than Hoeffding, but t-Test is even better, particularly when $\beta=3,20$ where the logging policy is near-deterministic ($w_{\max}$ is large).

Next, Figure~\ref{fig:error_rate} shows how frequently the estimated lower bounds fail to lower bound $V(\pi_e)$. Here, we say that a lower bound fails, if $\ipssimple (\pi_e; \calD_{val}) - f(\delta, \calD_{val}) \ge V(\pi_e)$. The black horizontal line represents the allowed error rate $\delta$. We observe that, in all scenarios, the bounds given by Hoeffding and Bernstein have an error rate of $0$, even if they are allowed to produce an error rate of $\delta$, meaning that these lower bounds are overly conservative. In contrast, the lower bound given by t-Test makes some errors, but the error rate is around $\delta$. Although the normality assumption might fail in OPE with small sample sizes, we empirically verify that t-Test produces a lower bound tighter than those of Hoeffding and Bernstein, and its error rate is around the allowed value.

\begin{table*}[ht] \label{tab:search_space}
\centering
\begin{minipage}{0.5\textwidth}
\centering
\caption{Hyperparameter search space ($\Theta$)}
\def\arraystretch{1.2}
\scalebox{1.05}{
\begin{tabular}{lll}
\hline
\textbf{Hyperparameters} & \textbf{Search Spaces} &  \\ \hline
$\beta$ & $[0.01, 100]$ &  \\ 
$\hat{\mu}$ & \{`LR', `RF'\} &  \\
\hline
\end{tabular}}
\end{minipage} \\
\begin{minipage}{0.5\textwidth}
\centering
\def\arraystretch{1.2}
\vspace{0.25in}
\scalebox{1.05}{
\begin{tabular}{lll}
\hline
\textbf{Machine Learning Models} & \textbf{Search Spaces} &  \\ \hline
\multirow{2}{*}{Logistic Regression (LR)} & $C \in [10^{-3}, 10^{3}]$ & \\
& $\text{l1\_ratio} \in \{0.1,0.2,\ldots,0.9\}$ &  \\ \hline
\multirow{3}{*}{Random Forest (RF)} & $\text{max\_depth} \in \{2,3,\ldots,32\}$ & \\
& $\text{min\_samples\_split} \in \{2, 3, \ldots, 32\}$ & \\ 
& $\text{max\_samples} \in \{0.1, 0.2, \ldots, 0.9\}$ & \\  
\hline
\end{tabular}}
\end{minipage}
\vskip 0.1in
\raggedright
\fontsize{9pt}{9pt}\selectfont \textit{Note}: 
The names of the hyperparameters correspond to those specified by the \textit{scikit-learn} package. For other hyperparameters, we use `sklearn.ensemble.RandomForectClassifier(n\_estimators=10)' and `sklearn.linear\_model.LogisticRegression(max\_iter=1000, penalty="elasticnet", solver="saga")'.
\end{table*}

\begin{table*}[ht] \label{tab:pi_b_optimality}
\centering
\vspace{0.15in}
\caption{Generalization performance and optimality of $\pi_0$ with varying $\beta_0$}
\def\arraystretch{1.2}
\scalebox{1.1}{
\begin{tabular}{c|cccccc}
\hline
& $\beta_0=-3$ & $\beta_0=0$ & $\beta_0=3$ & $\beta_0=10$ & $\beta_0=20$ \\ \hline
$V(\pi_0)$ & 0.412  & 0.501 & 0.580 & 0.677 & 0.719 \\
$V(\pi_0)/V(\pi^*)$ & 0.554 & 0.673 & 0.831 & 0.910 & 0.966 \\
\hline
\end{tabular}}
\vskip 0.1in
\centering
\fontsize{9pt}{9pt}\selectfont \textit{Note}: 
$V(\pi^*)$ is the best achievable performance in our data generating process. $V(\pi_0)/V(\pi^*)$ indicates the optimality of logging policy $\pi_0$.
\end{table*}

\end{document}